%% file: main.tex
\documentclass[conference]{IEEEtran}
\IEEEoverridecommandlockouts
\pdfoutput=1

\usepackage{color}
\usepackage{bm}
\usepackage{hyperref}
\usepackage{color}
\usepackage{balance}
\usepackage{amsthm}
\usepackage{times,amsmath}
\usepackage{amssymb}
\usepackage{subcaption}
\usepackage{caption}
\usepackage{cite}
\usepackage{url}
\usepackage[capitalise]{cleveref}
\usepackage{enumitem}
\usepackage{tikz, pgfplots}
\usepackage{graphicx}
\usepackage{subcaption}  
\usepackage{algorithm,algpseudocode}
\usepackage{moresize} 
\usetikzlibrary{arrows.meta, positioning, calc, fit, backgrounds}

\input{mysymbol.sty}

\pgfplotsset{compat=1.17}
\pgfplotstableset{col sep=semicolon}
\usepgfplotslibrary{groupplots}
\usepgfplotslibrary{fillbetween}

\tikzset{every mark/.append style={scale=1.6, solid}, font=\small}
\pgfplotsset{
    width=1\textwidth,
    height=5.5cm,
    legend style={
        font=\ssmall ,  
        inner xsep=1pt,
        inner ysep=1pt,
        nodes={inner sep=1pt}},
    legend cell align=left,
    every axis/.append style={line width=.5pt},
 	every axis plot/.append style={line width=1.5pt},
 	every axis y label/.append style={yshift=-4pt}
}

\begin{document}

\title{Unrolling Dynamic Programming via Graph Filters}

\author{Sergio~Rozada, Samuel Rey, Gonzalo Mateos, and~Antonio~G.~Marques

\thanks{S. Rozada, S. Rey, and A. G. Marques are with the Dep.
of Signal Theory, King Juan Carlos University, Madrid, Spain, \{sergio.rozada, samuel.rey.escudero, antonio.garcia.marques\}@urjc.es. G. Mateos is with the Dep. of ECE, University of Rochester, USA, gmateosb@ur.rochester.edu.}%

\thanks{Work partially funded by the Spanish AEI (10.13039/501100011033) grant PID2022-136887NB-I00, and the Community of Madrid via the Ellis Madrid Unit and grants URJC/CAM F861 and F1180 and TEC-2024/COM-89. Minor edits of this document were made with the assistance of ChatGPT.} }

\maketitle


\begin{abstract}

Dynamic programming (DP) is a fundamental tool used across many engineering fields. The main goal of DP is to solve Bellman’s optimality equations for a given Markov decision process (MDP). Standard methods like policy iteration exploit the fixed-point nature of these equations to solve them iteratively. However, these algorithms can be computationally expensive when the state-action space is large or when the problem involves long-term dependencies. Here we propose a new approach that unrolls and truncates 
policy iterations into a 
learnable 
parametric model dubbed BellNet, which we train to minimize the so-termed Bellman error from random value function initializations. Viewing the transition probability matrix of the MDP as the adjacency of a weighted directed graph, we draw insights from graph signal processing to 
interpret (and compactly re-parameterize) BellNet as a cascade of nonlinear graph filters. This fresh look
facilitates a concise, transferable, and unifying representation of policy and value iteration, with an explicit handle on complexity during inference. 
Preliminary experiments conducted in a grid-like environment demonstrate that BellNet can effectively approximate optimal policies in a fraction of the iterations required by classical methods.
\end{abstract}


\section{Introduction}

Dynamic programming (DP), recognized for its effectiveness in numerous engineering applications \cite{denardo2012dynamic}, is frequently modeled as a Markov decision process (MDP) \cite{puterman2014markov}. A central challenge in DP involves solving Bellman's equations (BEQs) to determine value functions (VFs), which represent cumulative long-term rewards. Since BEQs constitute fixed-point equations, DP commonly relies on iterative algorithms \cite{bertsekas2012dynamic, puterman2014markov}, wherein state transitions naturally induce a directed (di)graph structure. Despite their  effectiveness, these iterative methods face significant computational hurdles.
The number of required iterations until convergence grows rapidly with the size of the state-action space, and even more so in long-horizon problems.

To address these challenges, this work leverages algorithm unrolling \cite{gregor2010learning,chen2021graph} and graph signal processing (GSP) \cite{ortega2018graph,leus2023graph} to develop novel \emph{learnable} neural architectures 
that reduce the number of DP iterations. 
Unrolling techniques combine the 
interpretability of model-based algorithms with the flexibility of data-driven methods \cite{gregor2010learning, monga2021algorithm}. In unrolling, iterative algorithms are 
truncated to a finite sequence of update steps, eliminating conventional iterative loops. This sequential structure can then be mapped into a parametric model where iterations become layers, enabling particular elements of each update step to be learned directly from data rather than being prescribed~\cite{hadou2024robust}.

The iterative nature of DP  makes it particularly well-suited for unrolling strategies.
In our approach, we unroll the steps of policy and value iteration into a deep architecture termed BellNet, enabling a compact, data-driven alternative to traditional DP solvers. To design and customize BellNet, 
we draw on tools from GSP.
Specifically, we observe that each step in policy iteration can be formulated as a polynomial of the transition probability matrix, followed by a nonlinearity.
By interpreting the transition matrix as the adjacency of a weighted digraph, we identify this matrix polynomial as a graph filter~\cite{isufi2024graph}.
All in all, we find that unrolled policy iteration boils down to a cascade of nonlinear graph filters.
This GSP crossover enables the design of encompassing unrolled architectures that (i) require fewer (inference) steps to approximate policy iteration; and (ii) transfer across similar environments.
The summary of our contributions are
\begin{itemize}
    \item[\textbf{C1}] We introduce BellNet, an unrolled version of policy iteration structured as a cascade of nonlinear graph filters;
    \item[\textbf{C2}] We put forth a learning problem, where the filter coefficients are trained to minimize the so-termed Bellman error from random VF initializations; and
    \item[\textbf{C3}] We experimentally show in a grid-world setting that the learned BellNet model converges in significantly fewer iterations and generalizes well to similar environments.
\end{itemize}

\noindent \textbf{Prior work.} In reinforcement learning (RL), unrolling has been used in image-based settings~\cite{tamar2016value}, and to learn the MDP topology by interpreting the transition matrix as a graph~\cite{niu2018generalized, deac2020graph}. Unlike our work, existing approaches (a) focus on value iteration, a special case of the more general policy iteration framework; (b) address RL rather than DP, thus they estimate transition probabilities instead of exploiting the graph structure to design the unrolled architecture; and (c) target single tasks instead of enabling generalization across MDPs. 


Prior RL works have used GSP tools to improve algorithmic efficiency. For instance,~\cite{liu2019solving} postulates the VFs lie in a low-dimensional subspace induced by the state transition digraph; \cite{liu2019policy} estimates the optimal policy on a subset of states and extends it via graph interpolation; and~\cite{levorato2012reduced} applies graph reduction to simplify the decision process.
While effective, these methods are task-specific. In contrast, BellNet is task-agnostic and applicable across different MDPs.


Finally, a growing body of work in GSP investigates the properties of graph filters, e.g., permutation equivariance, stability, or transferability~\cite{gama2020stability,ruiz2021graph,levie2021transferability,cervino2023learning,rey2025redesigning,wang2025manifold}.
We empirically show that BellNet inherits some of these desirable properties, although a deeper theoretical analysis is left for future work.

\section{Preliminaries: Fundamentals of DP and GSP}

\noindent \textbf{DP.} In DP, we consider an MDP defined by the tuple $(\ccalS, \ccalA, \bbP, \bbR)$, where $\ccalS$ and $\ccalA$ are discrete state and action spaces, $\bbP \in [0,1]^{|\ccalS||\ccalA| \times |\ccalS|}$ is a known transition probability matrix whose rows, indexed by state-action pairs $(s,a)$, define distributions over next states $s’$, and $\bbR \in \reals^{|\ccalS| \times |\ccalA|}$ contains the rewards. Solving the MDP amounts to finding a policy $\pi:\ccalS\mapsto [0,1]^{|\ccalA|}$ that maximizes the VFs, defined as expected cumulative rewards.
A policy maps each state $s$ to a distribution over actions $a$, and the VF under $\pi$ is given by $Q^\pi(s, a) = \mathbb{E}_\pi \left[ \sum_{t=0}^\infty \gamma^t r_t \mid s_0 = s, a_0 = a \right]$, where $\gamma \in [0,1)$ is a discount factor and the instantaneous reward $r_t$ is the entry of $\bbR$ indexed by the state-action pair at time $t$. 
We arrange policy probabilities in the matrix $\bbPi \in [0,1]^{|\ccalS| \times |\ccalA|}$ and the VFs in $\bbQ_\pi \in \reals^{|\ccalS| \times |\ccalA|}$.
For convenience, we henceforth use the vectorizations $\bbr = \text{vec}(\bbR)$ and $\bbq_\pi = \text{vec}(\bbQ_\pi)$.

The BEQs characterize the VFs $\bbq_\pi$ for a fixed policy $\pi$~\cite{bellman1966dynamic}. Denoting $\bbP_\pi = \bbP (\bbI \odot \bbPi^\top )^\top$, where $\odot$ is the Khatri-Rao product and $\bbI$ the identity matrix, we have that
\begin{equation}\label{eq:bellman_equation}
    \bbq_\pi = \bbr + \gamma \bbP_\pi \bbq_\pi.
\end{equation}
This fixed-point linear system of equations can be solved iteratively. 
Iterating until convergence is referred to as \emph{policy evaluation} in DP parlance.
Greedy maximization of $\bbQ_\pi$ with respect to actions (columns) produces a new policy $\bbPi'$, i.e.,
\begin{equation}
    \label{eq::policy_improvement}
    \Pi'_{i, j} =
    \begin{cases}
        1 & \text{if } j = \arg\max_k Q_{i k}, \\
        0 & \text{otherwise.}
\end{cases}
\end{equation}
This step, known as \emph{policy improvement}, produces a policy $\bbPi'$ that is guaranteed to outperform $\bbPi$ in terms of the attained VFs~\cite{bertsekas2012dynamic}.
Crucially, if $\bbPi' = \bbPi$, then $\bbPi=\bbPi^\star$ is optimal, i.e., attains the maximum VFs $\bbQ_\pi=\bbQ^\star$ for all state-action pairs. This process underpins \emph{policy iteration}, an iterative method that alternates between policy evaluation and policy improvement to compute the optimal VFs and policy.

Interestingly, for the optimal VFs $\bbQ^\star$, it also holds that
\begin{equation}
    \label{eq::bellman_optimality}
    \bbq^\star = \bbr + \gamma \bbP \bbv^\star \quad \text{with} \quad v^\star_i = \max_k Q^\star_{i k}.
\end{equation}
This defines a nonlinear fixed-point system that can be solved iteratively through a procedure known as \emph{value iteration}. 
Value iteration is equivalent to performing one step of the policy evaluation iteration followed by policy improvement~\cite{bellman1966dynamic}.

\vspace{1mm}
\noindent
\textbf{GSP.}
A graph $\ccalG = (\ccalV, \ccalE)$ is defined by a set of $N$ nodes $\ccalV$ and a set of edges $\ccalE \subseteq \ccalV \times \ccalV$.
The connectivity of $\ccalG$ is captured by the sparse adjacency matrix $\bbA \in \mathbb{R}^{N \times N}$, where $A_{ij} \neq 0$ if and only if $(i,j) \in \ccalE$, and the entry $A_{ij}$ denotes the weight of the edge from node $i$ to node $j$.
A \emph{graph signal} is a function defined on the set of nodes, represented as a vector $\bbx \in \mathbb{R}^N$, where $x_i$ denotes the signal value at node $i$.

Graph filters are linear, topology-aware operators that process graph signals.
They can be expressed as matrix polynomials of the adjacency matrix $\bbA$~\cite{rey2023robust,isufi2024graph}, namely
\begin{eqnarray}\label{eq:graph_filter}
     &\bbH = \sum_{j=0}^{N-1} h_j \bbA^j,& 
\end{eqnarray}
where $\bbh = [h_0, \dots, h_{N-1}]^\top$ is the vector of filter coefficients. 
Since each power $\bbA^j$ encodes information about the $j$-hop neighborhood of $\ccalG$, the output $\bby = \bbH \bbx$ can be interpreted as a diffusion (or aggregation) of the input signal $\bbx$ across neighborhoods of increasing size, with the coefficients $h_j$ weighting the contribution from each $j$-hop component~\cite{segarra2017optimal}.

\section{Unrolling DP via GSP}


Algorithm unrolling is a foundational technique for infusing model-based inductive bias into data-driven learning~\cite{monga2021algorithm}.
Given an iterative algorithm, unrolling builds a parametric mapping, typically a neural network, by assigning each iteration to a corresponding block, such as a network layer.
The operations of the original algorithm are preserved and reinterpreted as layer-wise computations, enabling the model to learn algorithm-specific behavior from data.
Next, we unroll policy iteration and draw GSP connections in the process.  Each unrolled block 
consists of two main steps: policy evaluation, which involves solving~\eqref{eq:bellman_equation}, and policy improvement, where \eqref{eq::policy_improvement} is applied.\vspace{2pt}

\noindent \textbf{Policy evaluation.} 
The BEQ \eqref{eq:bellman_equation} is a linear system of equations. However, solving it directly is often impractical due to the large size of state-action spaces. Instead, one typically iterates by applying the right-hand-side (rhs) of \eqref{eq:bellman_equation} repeatedly until convergence is reached. Additional simplifications exploit structural properties of the MDP, such as linear dynamics \cite{lagoudakis2003least, geramifard2013tutorial}, low-rank structure \cite{agarwal2020flambe, rozada2024tensor, rozada2025solving}, or kernel-based representations \cite{ormoneit2002kernel, akiyama2025nonparametric}. Here, instead, we propose leveraging the graph structure of the MDP. 
To elucidate this connection, we expand the BEQ recursion as follows
%
\begin{align}
    \notag
    \bbq^{(k+1)} &= \bbr + \gamma \bbP_\pi \bbq^{(k)} = \bbr + \gamma \bbP_\pi \bbr + \gamma^2 \left(\bbP_\pi\right)^2 \bbq^{(k - 1)}  \\
    &=\hdots =\textstyle \sum_{j=0}^{k-1} \gamma^j \left(\bbP_\pi \right)^j \bbr + \gamma^k \left(\bbP_\pi\right)^k \bbq^{(0)}.   \label{eq::q_graph_filter}
\end{align}
This expression has two terms: an exponentially decaying bias $\bbb^{(k)} := \gamma^k \left(\bbP_\pi\right)^k \bbq^{(0)}$ that depends on the initial value $\bbq^{(0)}$; 
and a \emph{graph filter} $\bbH^{(k)} := \sum_{j=0}^{k-1} \gamma^j \left(\bbP_\pi \right)^j$ applied to $\bbr$. The latter characterization follows since $\bbH^{(k)}$ is a polynomial of $\bbP_\pi$, which represents the adjacency matrix of a weighted digraph $\ccalG$. The nodes are state-action pairs while the edge weights correspond to the Markovian transition probabilities $\bbP$ and the current policy $\bbPi$.
From this viewpoint, the powers of the discount factor $\gamma$ act as the filter coefficients in \eqref{eq:graph_filter}, i.e., $h_j=\gamma^j$. Consequently, policy evaluation can be interpreted as applying a graph filter to the reward signal. Due to the fixed-point theorem \cite{bertsekas2012dynamic}, 
an infinite-order filter is guaranteed to recover the true VF for policy $\pi$, so that $\bbq_\pi = \bbH^{(\infty)} \bbr + \bbb^{(\infty)} = \sum_{j=0}^\infty \gamma^j \left( \bbP_\pi \right)^j \bbr$.

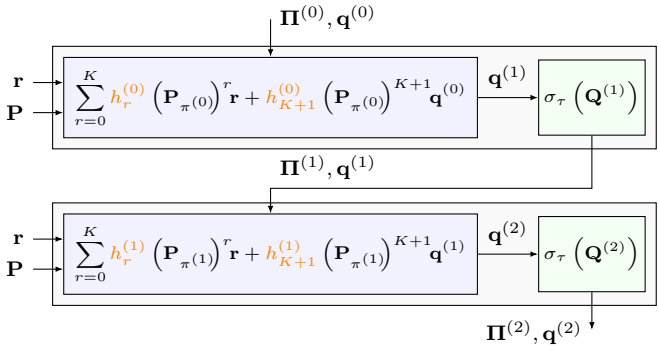
\begin{figure}[t]
    \centering
    \input{figures/policy_iteration_diagram.tex}
    \vspace{-0.5cm}
    \caption{BellNet schematic. A cascade of learnable graph filters and row-wise softmax nonlinearities that unrolls policy iteration.}
    \label{fig:unrolling}
    \vspace{-0.5cm}
\end{figure}

\vspace{1mm}
Moreover, our GSP perspective enables concrete simplifications of the proposed model. 
While the graph filter underlying policy evaluation is, in principle, of infinite degree,  an equivalent filter with limited degree exists. 

\begin{proposition}
\label{prop::finite_filter}
    The value function $\bbq^\pi$ under a fixed policy $\pi$ can be expressed as a finite-order graph filter
    \begin{equation}
        \label{eq::graph_filter_bounded}
       \textstyle \bbq_\pi = \sum_{j=0}^\infty \gamma^j (\bbP_\pi)^j \bbr = \sum_{j=0}^K \bar{h}_j (\bbP_\pi)^j \bbr,
    \end{equation}
    with $K \leq |\ccalS||\ccalA|$. If $\bbP_\pi$ is diagonalizable, then $K \leq |\ccalS|$.
\end{proposition}

\begin{proof}
By the Cayley–Hamilton theorem~\cite{horn2012matrix}, any matrix polynomial of $\bbP_\pi\in [0,1]^{|\ccalS||\ccalA| \times |\ccalS||\ccalA|}$ can be reparameterized as a polynomial of degree at most $K = |\ccalS||\ccalA|$. 
If $\bbP_\pi$ is diagonalizable, the degree of its minimal polynomial is at most $\operatorname{rank}(\bbP_\pi) = \operatorname{rank}(\bbP) = |\ccalS|$, so any polynomial of $\bbP_\pi$ can be expressed with order at most $K = |\ccalS|$.
\end{proof}


Beyond exact policy evaluation, our approach also encompasses \emph{approximate} policy evaluation via early stopping 
after a fixed number of iterations. Recall that value iteration corresponds to a \emph{single} application of the rhs of \eqref{eq:bellman_equation}. In any case, early stopping is equivalent to a graph filter of some order $K$ and fixed coefficients $h_j = \gamma^j$. This also introduces a non-vanishing bias term that must be accounted for [cf. \eqref{eq::q_graph_filter}]. Furthermore, the estimate $\hat \bbq_\pi$ may not converge to the true VF $\bbq_\pi$, so it must be reused to initialize the next policy evaluation under the updated policy $\bbPi’$.
Identity \eqref{eq::graph_filter_bounded} can be extended to this case by explicitly incorporating the bias term as
%
\begin{equation}
    \label{eq::graph_filter_bellnet}
   \textstyle \hat \bbq_\pi = \sum_{j=0}^{K} h_j (\bbP_\pi)^j \bbr + h_{K+1} \left( \bbP_\pi \right)^{K+1} \bbq^{(0)}.
\end{equation}

\noindent \textbf{Policy improvement.} As defined in \eqref{eq::policy_improvement}, policy improvement is a nonlinear row-wise max operation applied to $\bbQ_\pi$, analogous to max-pooling, selecting the maximum in each row. For differentiability, we replace the max operation with a softmax, as detailed in the next section.

\section{BellNet: Learning Policy Iteration}

Through the GSP lens, policy iteration is a cascade of nonlinear graph filtering operations that converge to the optimal VFs of the MDP.
This perspective motivates BellNet, our proposed unrolling of policy iteration to solve BEQs. 
BellNet is a deep architecture composed of $L+1$ layers. Each layer takes as input a VF vector $\bbq\in\reals^{|\ccalS||\ccalA|}$ and its associated softmax policy $\bbPi \in [0,1]^{|\ccalS| \times |\ccalA|}$, and outputs an enhanced VF vector and policy. The mapping between the input and output of the $l$-th layer is realized by a graph filter with learnable coefficients $\bbh^{(l)}$. Formally, let $\bar{\bbq}$ be the (possibly random) initial estimate of the VF, and let $\ccalH=\{\bbh^{(l)}\}_{l=0}^L$ collect all the learnable coefficients.
Then, BellNet, represented by the mapping $\bbPhi(\cdot; \ccalH)$, implements $\{\hat \bbq, \hat \bbPi \}:=\bbPhi(\bar{\bbq}; \ccalH)$ with $\hat \bbq\!=\!\bbq^{(L+1)}$, $\hat \bbPi\!=\!\bbPi^{(L+1)}$, $\bbq^{(0)}=\bar{\bbq}$, and layer-wise updates:
%
\begin{align}
   \notag
   \textstyle &\bbq^{(l + 1)} =\textstyle \sum_{j=0}^K h_j^{(l)} \left(\bbP_{\pi^{(l)}} \right)^j \bbr + h_{K+1}^{(l)} \left(\bbP_{\pi^{(l)}} \right)^{K+1} \bbq^{(l)} \\
    \notag
    &\bbPi^{(l+1)} = \sigma_\tau(\bbQ^{(l+1)}),\;\mathrm{with}\;[\sigma_\tau(\bbQ)]_{i j} = \frac{e^{Q_{i j} / \tau}}{\sum_{k=1}^{|\ccalA|} e^{Q_{i k} / \tau}},
\end{align}
for $l=0,\dots,L$, where $\bbQ^{(l)}=\text{unvec}(\bbq^{(l)})$, 
$\sigma_\tau$ is a row-wise softmax operator with temperature parameter $\tau$, and $\bbh^{(l)}=[h_0^{(l)},\dots,h_{K+1}^{(l)}]$ are the filter coefficients of the $l$-th layer. Each layer implements two 
reduced-order, parallel graph filters, sums their respective outputs, and then applies a softmax nonlinearity.
The BellNet model is illustrated in Fig. \ref{fig:unrolling}. Notably, setting $L=\infty$ and $K = \infty$ with $h^{(l)}_j = \gamma^j$ and replacing the softmax with the max operator recovers policy iteration. Similarly, setting $K = 1$ yields value iteration.\vspace{2pt}

\noindent \textbf{Learning.} To complete the approach, we formulate the optimization adopted to learn the filter coefficients $\ccalH$. 
The loss function is inspired by temporal difference (TD) methods \cite{sutton2018reinforcement, geist2013algorithmic}. We solve a sequential optimization problem that minimizes the Bellman error \cite{choi2006generalized, asadi2023td}, which is the discrepancy between the left and rhs of the optimal BEQ defined in \eqref{eq::bellman_optimality}. 
Specifically, with $n$ being an iteration index, we solve
\begin{equation}
    \label{eq::opt_problem}
    \ccalH_{[n+1]} = \text{argmin}_\ccalH \; \| \bbr + \gamma \bbP_{\bbPi_{[n]}} \bbq_{[n]} - \bbPhi(\bar{\bbq}, \ccalH)\|_2^2,
\end{equation}
where $\{\bbq{[n]},\bbPi_{[n]}\} := \bbPhi(\bar{\bbq}, \ccalH_{[n]})$. Note that $\{\bbq_{[n]},\bbPi_{[n]}\}$ depends on the current iterate $\ccalH_{[n]}$ and not on the optimized coefficients $\ccalH$. By slight abuse of notation, $\bbPhi$ in \eqref{eq::opt_problem} refers only to the VF output $\hat{\bbq}$. We also highlight that: (a) as customary in TD, for each $n$ we update the filter coefficients via gradient descent; (b) our DP method does not require data samples, but  
the transition probability matrix $\bbP$ instead; and (c) BellNet is initialized with an arbitrary VF $\bar{\bbq}$ and trained to converge to the optimal VF and policy regardless of $\bar{\bbq}$.\vspace{2pt}

\noindent \textbf{Transferability.}
Graph filters are permutation-equivariant and transferable to larger graphs from a convergent sequence~\cite{ruiz2021graph}, making them particularly well suited to generalize across related problems. In our DP context, this property can be leveraged to train BellNet on a single MDP and deploy it on other similar or larger MDPs. Doing so yields solutions faster than evaluating policies from scratch, as we demonstrate numerically in Section \ref{sec:numerical}.
Moreover, the vanilla BellNet described so far operates with a fixed unrolling depth and distinct parameters per block. An attractive alternative is to \emph{share} weights across blocks. Although weight sharing admittedly reduces expressiveness, it markedly decreases the number of learnable parameters~\cite{nowlan1992simplifying, chen2021graph}. Crucially, it allows the same block to be reused as many times as desired during inference--exceeding the original training depth to enable efficient, scalable transfer as well as to delineate favorable complexity versus policy approximation tradeoffs.

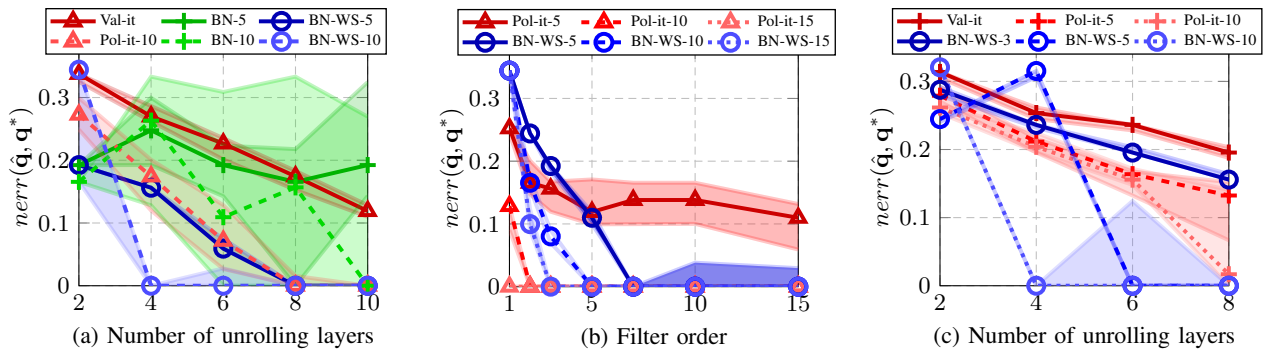
\begin{figure*}[!t]
	\centering
	\begin{subfigure}{0.28\textwidth}
		\centering
		 \input{figures/exp_unrolls}\label{fig:exp_a}
	\end{subfigure}
    \hspace{.4cm}
	\begin{subfigure}{0.28\textwidth}
		\centering
        \input{figures/exp_K}
	\end{subfigure}
     \hspace{.4cm}
	\begin{subfigure}{0.28\textwidth}
		\centering
		\input{figures/exp_transfer}
	\end{subfigure}
	\caption{Evaluation of BellNet across different scenarios. We report the median error of the estimated $\hbq$, computed as in~\eqref{eq:nerr}, over 15 realizations. a) Shows the error as $L$ increases; b) illustrates the error as $K$ increases; and c) evaluates the transferability capacity of BellNet.}
    \label{fig:exp}
\end{figure*}

\begin{figure}[t]
    \centering
    \includegraphics[width=.85\linewidth]{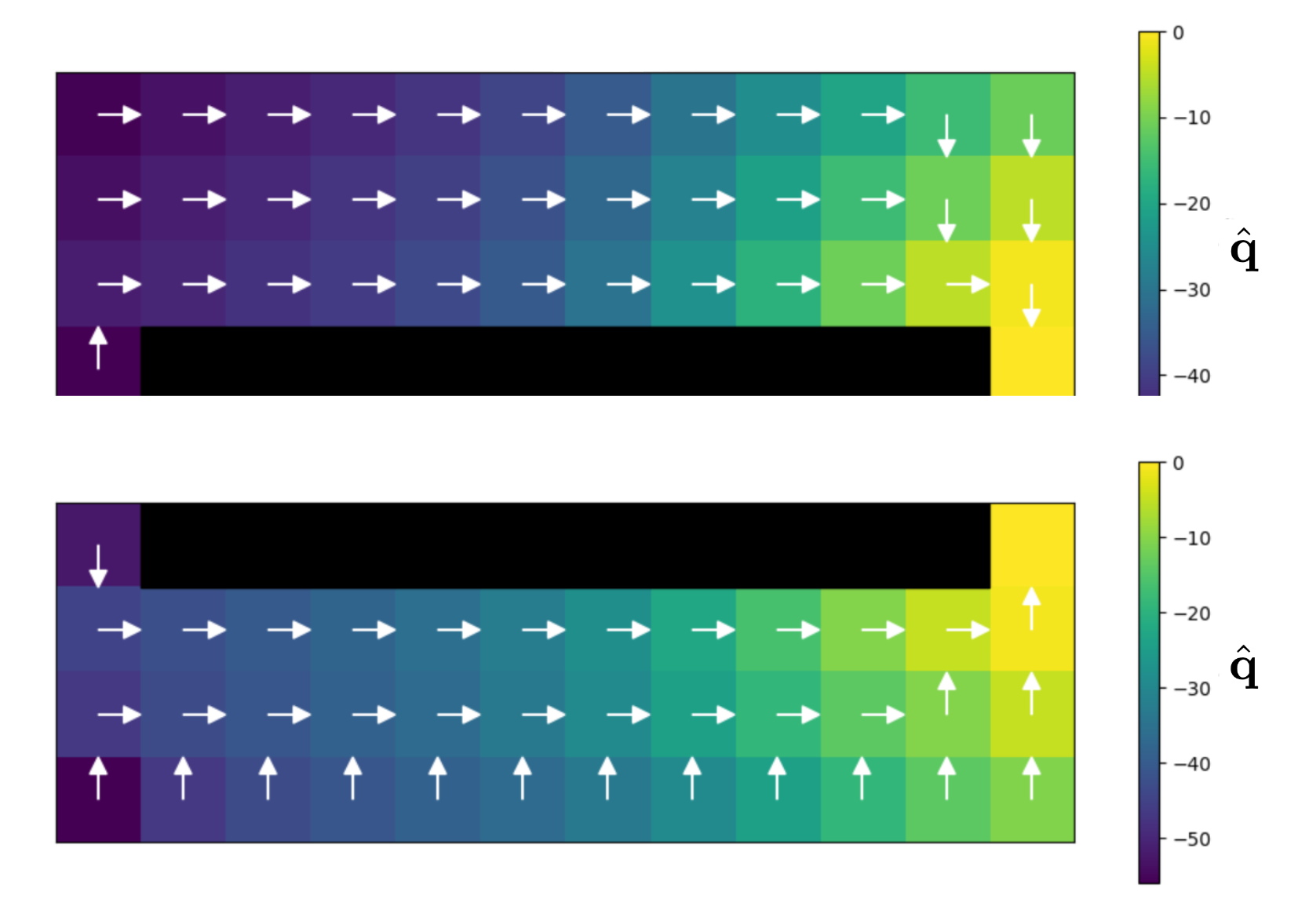}
    \caption{Cliff walking environment (top) and its mirrored version (bottom). 
    Cliff regions are shown in black; arrows indicate the policy learned by BellNet, and the color map represents the corresponding VFs. 
    BellNet is trained on the top environment, while the policy in the bottom environment is inferred without retraining.}
    \label{fig:grid_world}
    \vspace{-0.5cm}
\end{figure}


\section{Numerical Results and Concluding Remarks}\label{sec:numerical}

We assess the performance of BellNet in the cliff walking environment, a grid-world setup where the goal is to reach a target location in the minimum number of steps without falling off the grid.
The state space $\ccalS$ corresponds to positions on the grid, and the action space consists of moving up, down, left, or right.
Two instances of this environment are depicted in Fig.~\ref{fig:grid_world}.
Simulations are conducted using the Gymnasium library \cite{brockman2016openai, towers2024gymnasium}, and the code is available on GitHub\footnote{ \url{https://github.com/sergiorozada12/rl-unrolling}} for reproducibility.

We compute the true VF $\bbq^*$ using policy iteration with sufficiently many policy evaluation and improvement steps, and report the normalized error defined as
\begin{equation}\label{eq:nerr}
    nerr(\hbq, \bbq^*) = \big\| \hbq/\|\hbq\|_2 - \bbq^*/\|\bbq^*\|_2\big\|_2^2.
\end{equation}
%
Figure~\ref{fig:exp} 
depicts the median error along with the interquartile range (between the 25th and 75th percentiles), computed over 15 random realizations.
We compare the performance of BellNet with and without weight sharing (denoted ``BN-WS'' and ``BN'' in the legend), as well as value iteration (``Val-it'') and policy iteration (``Pol-it''), across multiple scenarios.\vspace{2pt}


\noindent
\textbf{Test case 1 (Depth).}
We first examine how increasing the number of unrolling layers (equivalently, the number of policy improvement steps for ``Val-it'' and ``Pol-it'') influences performance.
Figure~\ref{fig:exp}a shows results using filter orders 5 and 10 for ``BN'', and 10 policy evaluation updates in ``Pol-it''.
Apparently, the weight sharing strategy leads to better performance with lower variance, whereas distinct filter coefficients 
results in more unstable behavior.
Moreover, BellNet consistently outperforms policy iteration, recovering the optimal policy with only 4 layers compared to 10 required by ``Pol-it''.\vspace{2pt}

\noindent
\textbf{Test case 2 (Filter order).}
Next, we investigate the role of the filter order in the performance of BellNet. Figure~\ref{fig:exp}b shows the error of ``Pol-it'' and ``BN-WS'' as the number of policy evaluation steps and, correspondingly, the filter order, increases as indicated on the x-axis.
We consider 5, 10, and 15 policy improvement steps for ``Pol-it'' and the same number of unrolling layers for ``BN-WS''.
As expected, 
we find that a higher filter order improves the performance of ``BN-WS'', with a smaller order being sufficient when the number of unrolling layers increases.
Interestingly, this is not the case for ``Pol-it'', where the number of policy improvement steps has a greater impact than the number of evaluation steps in this setting.
Consistent with the previous experiment, these results highlight how BellNet 
outperforms ``Pol-it'' when the number of policy improvement steps is moderately small.\vspace{2pt}

\noindent
\textbf{Test case 3 (Transferability).}
The last experiment inspects BellNet's transferability properties. 
We train BellNet in the original grid-world setting used in previous test cases, and then use it to predict the optimal policy in a modified environment where the positions of the cliffs, origin, and destination have changed.
As shown in Fig.~\ref{fig:exp}c, BellNet successfully predicts the optimal policy in the new environment without requiring retraining.
For comparison, we compute the optimal policy in the modified environment using value iteration and policy iteration with 5 and 10 policy evaluation steps.
We observe that the error in the estimated $\hbq$ decreases as the number of layers (indicated on the x-axis) increases, or, when higher-order filters are used.
Notably, BellNet outperforms the classical baselines when both the number of unrolling layers and the filter order are sufficiently large.
Overall, these preliminary results demonstrate that BellNet not only offers a novel approach to estimating the optimal policy, but also generalizes effectively to other related environments not seen during training.

\newpage

\balance
\bibliographystyle{IEEEbib}
\bibliography{myIEEEabrv,biblio}

\end{document}

%% file: figures/policy_iteration_diagram.tex
\begin{tikzpicture}[
    box/.style={
        draw=black!100,
        inner sep=6pt,
        minimum height=1.2cm,
        align=center
    },
    linbox/.style={
        box,
        font=\scriptsize,
        inner sep=4pt,
        fill=blue!5,
        minimum height=1cm,
        minimum width=0.5cm
    },
    softbox/.style={
        box,
        font=\scriptsize,
        inner sep=2pt,
        fill=green!5,
        minimum height=1cm,
        minimum width=0.5cm
    },
    blockbox/.style={
        draw,
        inner sep=4pt,
        fill=gray!5,
        minimum height=1cm,
        minimum width=0.5cm
    },
    smalltip/.style={->, >={Latex[scale=0.6]}},
    every node/.style={font=\footnotesize},
    >={Stealth[length=4pt, width=6pt]},
    node distance=1.0cm and 0.8cm
]

\node[linbox] (poly0) {$ \displaystyle \sum_{r=0}^K {\color{orange}h_r^{(0)}} \left(\bbP_{\pi^{(0)}} \!\! \right)^r \!\! \bbr + {\color{orange}h_{K+1}^{(0)}} \left(\bbP_{\pi^{(0)}}\!\!\right)^{K+1} \!\! \bbq^{(0) }$};
\node[softbox, right=of poly0] (softmax0) {$\sigma_\tau\left(\bbQ^{(1)}\right)$};
\begin{pgfonlayer}{background}
\node[blockbox, fit=(poly0)(softmax0)] {};
\end{pgfonlayer}

\node[linbox, below=of poly0] (polyk) {$ \displaystyle \sum_{r=0}^K {\color{orange}h_r^{(1)}}\left(\bbP_{\pi^{(1)}} \!\! \right)^r \!\! \bbr + {\color{orange}h_{K+1}^{(1)}} \left(\bbP_{\pi^{(1)}}\!\!\right)^{K+1} \!\! \bbq^{(1) }$};
\node[softbox, right=of polyk] (softmaxk) {$\sigma_\tau\left(\bbQ^{(2)} \right)$};
\begin{pgfonlayer}{background}
\node[blockbox, fit=(polyk)(softmaxk)] {};
\end{pgfonlayer}

\node[left=0.4cm of poly0, yshift=0.2cm] (r1) {$\bbr$};
\node[left=0.4cm of poly0, yshift=-0.2cm] (P1) {$\bbP$};

\draw[smalltip] (r1.east) -- ++(0.3,0) |- ([yshift=0.2cm]poly0.west);
\draw[smalltip] (P1.east) -- ++(0.3,0) |- ([yshift=-0.2cm]poly0.west);

\node[left=0.4cm of polyk, yshift=0.2cm] (r2) {$\bbr$};
\node[left=0.4cm of polyk, yshift=-0.2cm] (P2) {$\bbP$};

\draw[smalltip] (r2.east) -- ++(0.3,0) |- ([yshift=0.2cm]polyk.west);
\draw[smalltip] (P2.east) -- ++(0.3,0) |- ([yshift=-0.2cm]polyk.west);

\draw[smalltip] 
    ($(poly0.north)+(0,0.5)$) -- ($(poly0.north)+(0,0)$)
    node[midway, above right] {$\bbPi^{(0)}, \bbq^{(0)}$};

\draw[smalltip] (poly0) -- node[above] {$\bbq^{(1)}$} (softmax0);

\draw[smalltip] 
    ($(softmax0.south)+(0,0)$) 
    -- ++(0,-0.7) 
    -| ($(polyk.north)+(0,0)$)
    node[midway, above right] {$\bbPi^{(1)}, \bbq^{(1)}$};
    
\draw[smalltip] (polyk) -- node[above] {$\bbq^{(2)}$} (softmaxk);

\draw[smalltip] 
    ($(softmaxk.south)+(0,0)$) -- ($(softmaxk.south)+(0,-0.5)$)
    node[midway, below left] {$\bbPi^{(2)}, \bbq^{(2)}$};

\end{tikzpicture}

%% file: figures/exp_unrolls.tex
\begin{tikzpicture}[baseline,scale=.95]

\pgfplotstableread{data/n_unrolls_all_data_med_err.csv}\errtable
\pgfplotstableread{data/n_unrolls_all_data_prctile75.csv}\prcttop
\pgfplotstableread{data/n_unrolls_all_data_prctile25.csv}\prctbot

\pgfmathsetmacro{\opacity}{0.2}
\pgfmathsetmacro{\contourop}{0.25}

\begin{axis}[
    xlabel={(a) Number of unrolling layers},
    xmin=2,
    xmax=10,
    ylabel={$nerr(\hbq, \bbq^*)$},
    ymin = 0,
    ymax = .38,
    grid style=densely dashed,
    grid=both,
    legend style={
        at={(.5, 1.02)},
        anchor=south},
    legend columns=3,
    width=160,
    height=140,
    ]

   \addplot [red!80!black, name path = val-it-bot, opacity=\contourop, forget plot] table [x=xaxis, y=val-it] \prctbot;
    \addplot [red!80!black, name path = val-it-top, opacity=\contourop, forget plot] table [x=xaxis, y=val-it] \prcttop;
    \addplot[red!90!white, fill opacity=\opacity, forget plot] fill between[of=val-it-bot and val-it-top];
    \addplot[red!80!black, mark=triangle, solid] table [x=xaxis, y=val-it] {\errtable};


    \addplot [green!70!black, name path = unr-K5-bot, opacity=\contourop, forget plot] table [x=xaxis, y=unr-K5] \prctbot;
    \addplot [green!70!black, name path = unr-K5-top, opacity=\contourop, forget plot] table [x=xaxis, y=unr-K5] \prcttop;
    \addplot[green!70!black, fill opacity=\opacity, forget plot] fill between[of=unr-K5-bot and unr-K5-top];
    \addplot[green!70!black, mark=+, solid] table [x=xaxis, y=unr-K5] {\errtable};

    \addplot [blue!80!white, name path = unr-K5-WS-bot, opacity=\contourop, forget plot] table [x=xaxis, y=unr-K5-WS] \prctbot;
    \addplot [blue!80!white, name path = unr-K5-WS-top, opacity=\contourop, forget plot] table [x=xaxis, y=unr-K5-WS] \prcttop;
    \addplot[blue!70!black, fill opacity=\opacity, forget plot] fill between[of=unr-K5-WS-bot and unr-K5-WS-top];
    \addplot[blue!70!black, mark=o , solid] table [x=xaxis, y=unr-K5-WS] {\errtable};
    
    \addplot [red!70!white, name path = pol-it-10eval-bot, opacity=\contourop, forget plot] table [x=xaxis, y=pol-it-10eval] \prctbot;
    \addplot [red!70!white, name path = pol-it-10eval-top, opacity=\contourop, forget plot] table [x=xaxis, y=pol-it-10eval] \prcttop;
    \addplot[red!60!white, fill opacity=\opacity, forget plot] fill between[of=pol-it-10eval-bot and pol-it-10eval-top];
    \addplot[red!70!white, mark=triangle, dashed] table [x=xaxis, y=pol-it-10eval] {\errtable}; 

    \addplot [green!85!black, name path = unr-K10-bot, opacity=\contourop, forget plot] table [x=xaxis, y=unr-K10] \prctbot;
    \addplot [green!85!black, name path = unr-K10-top, opacity=\contourop, forget plot] table [x=xaxis, y=unr-K10] \prcttop;
    \addplot[green!95!black, fill opacity=\opacity, forget plot] fill between[of=unr-K10-bot and unr-K10-top];
    \addplot[green!85!black, mark=+, dashed] table [x=xaxis, y=unr-K10] {\errtable};

    \addplot [blue!50!white, name path = unr-K10-WS-bot, opacity=\contourop, forget plot] table [x=xaxis, y=unr-K10-WS] \prctbot;
    \addplot [blue!50!white, name path = unr-K10-WS-top, opacity=\contourop, forget plot] table [x=xaxis, y=unr-K10-WS] \prcttop;
    \addplot[blue!70!white, fill opacity=\opacity, forget plot] fill between[of=unr-K10-WS-bot and unr-K10-WS-top];
    \addplot[blue!70!white, mark=o , dashed] table [x=xaxis, y=unr-K10-WS] {\errtable};


    \legend{Val-it,  BN-5, BN-WS-5, Pol-it-10, BN-10, BN-WS-10}
\end{axis}
\end{tikzpicture}

%% file: figures/exp_K.tex
\begin{tikzpicture}[baseline,scale=.95]

\pgfplotstableread{data/K_pol_vs_ws_all_data_med_err.csv}\errtable
\pgfplotstableread{data/K_pol_vs_ws_unr_all_data_prctile25.csv}\prcttop
\pgfplotstableread{data/K_pol_vs_ws_all_data_prctile75.csv}\prctbot

\pgfmathsetmacro{\opacity}{0.3}
\pgfmathsetmacro{\contourop}{0.25}

\begin{axis}[
    xlabel={(b) Filter order},
    xmin=1,
    xmax=15,
    xtick = {1, 5, 10, 15},
    ylabel={$nerr(\hbq, \bbq^*)$},
    ymin = 0,
    ymax = .38,
    grid style=densely dashed,
    grid=both,
    legend style={
        at={(.5, 1.02)},
        anchor=south},
    legend columns=3,
    width=160,
    height=140,
    ]

    \addplot [red!90!white, name path = pol-it-5eval-bot, opacity=\contourop, forget plot] table [x=xaxis, y=pol-it-5eval] \prctbot;
    \addplot [red!90!white, name path = pol-it-5eval-top, opacity=\contourop, forget plot] table [x=xaxis, y=pol-it-5eval] \prcttop;
    \addplot[red!90!white, fill opacity=\opacity, forget plot] fill between[of=pol-it-5eval-bot and pol-it-5eval-top];
    \addplot[red!80!black, mark=triangle, solid] table [x=xaxis, y=pol-it-5eval] {\errtable};
    
    \addplot [red!70!white, name path = pol-it-10eval-bot, opacity=\contourop, forget plot] table [x=xaxis, y=pol-it-10eval] \prctbot;
    \addplot [red!70!white, name path = pol-it-10eval-top, opacity=\contourop, forget plot] table [x=xaxis, y=pol-it-10eval] \prcttop;
    \addplot[red!70!white, fill opacity=\opacity, forget plot] fill between[of=pol-it-10eval-bot and pol-it-10eval-top];
    \addplot[red, mark=triangle, dashed] table [x=xaxis, y=pol-it-10eval] {\errtable};
    
    \addplot [red!50!white, name path = pol-it-15eval-bot, opacity=\contourop, forget plot] table [x=xaxis, y=pol-it-15eval] \prctbot;
    \addplot [red!50!white, name path = pol-it-15eval-top, opacity=\contourop, forget plot] table [x=xaxis, y=pol-it-15eval] \prcttop;
    \addplot[red!50!white, fill opacity=\opacity, forget plot] fill between[of=pol-it-15eval-bot and pol-it-15eval-top];
    \addplot[red!60!white, mark=triangle, dotted] table [x=xaxis, y=pol-it-15eval] {\errtable};

    \addplot [blue!80!white, name path = unr-5unrolls-WS-bot, opacity=\contourop, forget plot] table [x=xaxis, y=unr-5unrolls-WS] \prctbot;
    \addplot [blue!80!white, name path = unr-5unrolls-WS-top, opacity=\contourop, forget plot] table [x=xaxis, y=unr-5unrolls-WS] \prcttop;
    \addplot[blue!70!black, fill opacity=\opacity, forget plot] fill between[of=unr-5unrolls-WS-bot and unr-5unrolls-WS-top];
    \addplot[blue!70!black, mark=o , solid] table [x=xaxis, y=unr-5unrolls-WS] {\errtable};

    \addplot [blue!65!white, name path = unr-10unrolls-WS-bot, opacity=\contourop, forget plot] table [x=xaxis, y=unr-10unrolls-WS] \prctbot;
    \addplot [blue!65!white, name path = unr-10unrolls-WS-top, opacity=\contourop, forget plot] table [x=xaxis, y=unr-10unrolls-WS] \prcttop;
    \addplot[blue!90!white, fill opacity=\opacity, forget plot] fill between[of=unr-5unrolls-WS-bot and unr-5unrolls-WS-top];
    \addplot[blue, mark=o , dashed] table [x=xaxis, y=unr-10unrolls-WS] {\errtable};

    \addplot [blue!50!white, name path = unr-15unrolls-WS-bot, opacity=\contourop, forget plot] table [x=xaxis, y=unr-15unrolls-WS] \prctbot;
    \addplot [blue!50!white, name path = unr-15unrolls-WS-top, opacity=\contourop, forget plot] table [x=xaxis, y=unr-15unrolls-WS] \prcttop;
    \addplot[blue!70!white, fill opacity=\opacity, forget plot] fill between[of=unr-15unrolls-WS-bot and unr-15unrolls-WS-top];
    \addplot[blue!70!white, mark=o , dotted] table [x=xaxis, y=unr-15unrolls-WS] {\errtable};

    \legend{Pol-it-5, Pol-it-10, Pol-it-15, BN-WS-5, BN-WS-10, BN-WS-15}
\end{axis}
\end{tikzpicture}

%% file: figures/exp_transfer.tex
\begin{tikzpicture}[baseline,scale=.95]

\pgfplotstableread{data/transfer_all_data_med_err.csv}\errtable
\pgfplotstableread{data/transfer_all_data_prctile25.csv}\prcttop
\pgfplotstableread{data/transfer_all_data_prctile75.csv}\prctbot

\pgfmathsetmacro{\opacity}{0.2}
\pgfmathsetmacro{\contourop}{0.25}

\begin{axis}[
    xlabel={(c) Number of unrolling layers},
    xmin=2,
    xmax=8,
    ylabel={$nerr(\hbq, \bbq^*)$},
    ymin = 0,
    ymax = .35,
    grid style=densely dashed,
    grid=both,
    legend style={
        at={(.5, 1.02)},
        anchor=south},
    legend columns=3,
    width=160,
    height=140,
    ]

    \addplot [red!90!white, name path = val-it-bot, opacity=\contourop, forget plot] table [x=xaxis, y=val-it] \prctbot;
    \addplot [red!90!white, name path = val-it-top, opacity=\contourop, forget plot] table [x=xaxis, y=val-it] \prcttop;
    \addplot[red!90!white, fill opacity=\opacity, forget plot] fill between[of=val-it-bot and val-it-top];
    \addplot[red!80!black, mark=+, solid] table [x=xaxis, y=val-it] {\errtable};

    \addplot [red!70!white, name path = pol-it-5eval-bot, opacity=\contourop, forget plot] table [x=xaxis, y=pol-it-5eval] \prctbot;
    \addplot [red!70!white, name path = pol-it-5eval-top, opacity=\contourop, forget plot] table [x=xaxis, y=pol-it-5eval] \prcttop;
    \addplot[red!70!white, fill opacity=\opacity, forget plot] fill between[of=pol-it-5eval-bot and pol-it-5eval-top];
    \addplot[red, mark=+, dashed] table [x=xaxis, y=pol-it-5eval] {\errtable};

    \addplot [red!50!white, name path = pol-it-10eval-bot, opacity=\contourop, forget plot] table [x=xaxis, y=pol-it-10eval] \prctbot;
    \addplot [red!50!white, name path = pol-it-10eval-top, opacity=\contourop, forget plot] table [x=xaxis, y=pol-it-10eval] \prcttop;
    \addplot[red!50!white, fill opacity=\opacity, forget plot] fill between[of=pol-it-10eval-bot and pol-it-10eval-top];
    \addplot[red!60!white, mark=+, dotted] table [x=xaxis, y=pol-it-10eval] {\errtable};

    \addplot [blue!80!white, name path = unr-K3-WS-bot, opacity=\contourop, forget plot] table [x=xaxis, y=unr-K3-WS] \prctbot;
    \addplot [blue!80!white, name path = unr-K3-WS-top, opacity=\contourop, forget plot] table [x=xaxis, y=unr-K3-WS] \prcttop;
    \addplot[blue!80!white, fill opacity=\opacity, forget plot] fill between[of=unr-K3-WS-bot and unr-K3-WS-top];
    \addplot[blue!70!black, mark=o , solid] table [x=xaxis, y=unr-K3-WS] {\errtable};

    \addplot [blue!90!white, name path = unr-K5-WS-bot, opacity=\contourop, forget plot] table [x=xaxis, y=unr-K5-WS] \prctbot;
    \addplot [blue!90!white, name path = unr-K5-WS-top, opacity=\contourop, forget plot] table [x=xaxis, y=unr-K5-WS] \prcttop;
    \addplot[blue!65!white, fill opacity=\opacity, forget plot] fill between[of=unr-K5-WS-bot and unr-K5-WS-top];
    \addplot[blue, mark=o , dashed] table [x=xaxis, y=unr-K5-WS] {\errtable};

    \addplot [blue!50!white, name path = unr-K10-WS-bot, opacity=\contourop, forget plot] table [x=xaxis, y=unr-K10-WS] \prctbot;
    \addplot [blue!50!white, name path = unr-K10-WS-top, opacity=\contourop, forget plot] table [x=xaxis, y=unr-K10-WS] \prcttop;
    \addplot[blue!70!white, fill opacity=\opacity, forget plot] fill between[of=unr-K10-WS-bot and unr-K10-WS-top];
    \addplot[blue!70!white, mark=o , dotted] table [x=xaxis, y=unr-K10-WS] {\errtable};

    \legend{Val-it, Pol-it-5, Pol-it-10, BN-WS-3, BN-WS-5, BN-WS-10}
\end{axis}
\end{tikzpicture}

%% file: main.bbl
\begin{thebibliography}{10}

\bibitem{denardo2012dynamic}
E.~V. Denardo,
\newblock {\em Dynamic Programming: {M}odels and Applications},
\newblock Courier Corporation, 2012.

\bibitem{puterman2014markov}
M.~L. Puterman,
\newblock {\em Markov Decision Processes: {D}iscrete Stochastic Dynamic Programming},
\newblock John Wiley \& Sons, 2014.

\bibitem{bertsekas2012dynamic}
D.~Bertsekas,
\newblock {\em Dynamic Programming and Optimal Control: {V}olume I}, vol.~4,
\newblock Athena Scientific, 2012.

\bibitem{gregor2010learning}
K.~Gregor and Y.~LeCun,
\newblock ``Learning fast approximations of sparse coding,''
\newblock in {\em Int. Conf. Mach. Learning}, 2010, pp. 399--406.

\bibitem{chen2021graph}
S.~Chen, Y.~C. Eldar, and L.~Zhao,
\newblock ``Graph unrolling networks: {I}nterpretable neural networks for graph signal denoising,''
\newblock {\em {IEEE} Trans. Signal Process.}, vol. 69, pp. 3699--3713, 2021.

\bibitem{ortega2018graph}
A.~Ortega, P.~Frossard, J.~Kova{\v{c}}evi{\'c}, J.~M.~F. Moura, and P.~Vandergheynst,
\newblock ``Graph signal processing: {O}verview, challenges, and applications,''
\newblock {\em Proc. {IEEE}}, vol. 106, no. 5, pp. 808--828, 2018.

\bibitem{leus2023graph}
G.~Leus, A.~G. Marques, J.~M.~F. Moura, A.~Ortega, and D.~I. Shuman,
\newblock ``Graph signal processing: {H}istory, development, impact, and outlook,''
\newblock {\em {IEEE} Signal Process. Mag.}, vol. 40, no. 4, pp. 49--60, 2023.

\bibitem{monga2021algorithm}
V.~Monga, Y.~Li, and Y.~C. Eldar,
\newblock ``Algorithm unrolling: {I}nterpretable, efficient deep learning for signal and image processing,''
\newblock {\em {IEEE} Signal Process. Mag.}, vol. 38, no. 2, pp. 18--44, 2021.

\bibitem{hadou2024robust}
S.~Hadou, N.~NaderiAlizadeh, and A.~Ribeiro,
\newblock ``Robust stochastically-descending unrolled networks,''
\newblock {\em {IEEE} Trans. Signal Process.}, vol. 72, pp. 5484--5499, 2024.

\bibitem{isufi2024graph}
E.~Isufi, F.~Gama, D.~I. Shuman, and S.~Segarra,
\newblock ``Graph filters for signal processing and machine learning on graphs,''
\newblock {\em {IEEE} Trans. Signal Process.}, vol. 72, pp. 4745--4781, 2024.

\bibitem{tamar2016value}
A.~Tamar, Y.~Wu, G.~Thomas, S.~Levine, and P.~Abbeel,
\newblock ``Value iteration networks,''
\newblock in {\em Conf. Neural Inform. Process. Syst.}, 2016, vol.~29.

\bibitem{niu2018generalized}
S.~Niu, S.~Chen, H.~Guo, C.~Targonski, M.~Smith, and J.~Kova{\v{c}}evi{\'c},
\newblock ``Generalized value iteration networks: {L}ife beyond lattices,''
\newblock in {\em {AAAI} Conf. Artificial Intell.}, 2018, vol.~32.

\bibitem{deac2020graph}
A.~Deac, P.-L. Bacon, and J.~Tang,
\newblock ``Graph neural induction of value iteration,''
\newblock {\em arXiv preprint arXiv:2009.12604}, 2020.

\bibitem{liu2019solving}
L.~Liu, A.~Chattopadhyay, and U.~Mitra,
\newblock ``On solving {MDP}s with large state space: {E}xploitation of policy structures and spectral properties,''
\newblock {\em {IEEE} Trans. Commun.}, vol. 67, no. 6, pp. 4151--4165, 2019.

\bibitem{liu2019policy}
L.~Liu and U.~Mitra,
\newblock ``Policy sampling and interpolation for wireless networks: {A} graph signal processing approach,''
\newblock in {\em Proc. IEEE Global Commun. Conf.}, 2019, pp. 1--6.

\bibitem{levorato2012reduced}
M.~Levorato, S.~Narang, U.~Mitra, and A.~Ortega,
\newblock ``Reduced dimension policy iteration for wireless network control via multiscale analysis,''
\newblock in {\em Proc. IEEE Global Commun. Conf.}, 2012, pp. 3886--3892.

\bibitem{gama2020stability}
F.~Gama, J.~Bruna, and A.~Ribeiro,
\newblock ``Stability properties of graph neural networks,''
\newblock {\em {IEEE} Trans. Signal Process.}, vol. 68, pp. 5680--5695, 2020.

\bibitem{ruiz2021graph}
L.~Ruiz, F.~Gama, and A.~Ribeiro,
\newblock ``Graph neural networks: {A}rchitectures, stability, and transferability,''
\newblock {\em Proc. {IEEE}}, vol. 109, no. 5, pp. 660--682, 2021.

\bibitem{levie2021transferability}
R.~Levie, W.~Huang, L.~Bucci, M.~Bronstein, and G.~Kutyniok,
\newblock ``Transferability of spectral graph convolutional neural networks,''
\newblock {\em J. Mach. Learning Res.}, vol. 22, no. 272, pp. 1--59, 2021.

\bibitem{cervino2023learning}
J.~Cervino, L.~Ruiz, and A.~Ribeiro,
\newblock ``Learning by transference: {T}raining graph neural networks on growing graphs,''
\newblock {\em {IEEE} Trans. Signal Process.}, vol. 71, pp. 233--247, 2023.

\bibitem{rey2025redesigning}
S.~Rey, M.~Navarro, V.~M. Tenorio, S.~Segarra, and A.~G. Marques,
\newblock ``Redesigning graph filter-based {GNN}s to relax the homophily assumption,''
\newblock in {\em {IEEE} Int. Conf. Acoust., Speech and Signal Process.}, 2025, pp. 1--5.

\bibitem{wang2025manifold}
Z.~Wang, J.~Cervino, and A.~Ribeiro,
\newblock ``A manifold perspective on the statistical generalization of graph neural networks,''
\newblock in {\em Int. Conf. Learn. Representations}, 2025.

\bibitem{bellman1966dynamic}
R.~Bellman,
\newblock ``Dynamic programming,''
\newblock {\em Science}, vol. 153, no. 3731, pp. 34--37, 1966.

\bibitem{rey2023robust}
S.~Rey, V.~M. Tenorio, and A.~G. Marques,
\newblock ``Robust graph filter identification and graph denoising from signal observations,''
\newblock {\em {IEEE} Trans. Signal Process.}, vol. 71, pp. 3651--3666, 2023.

\bibitem{segarra2017optimal}
S.~Segarra, A.~G. Marques, and A.~Ribeiro,
\newblock ``Optimal graph-filter design and applications to distributed linear network operators,''
\newblock {\em {IEEE} Trans. Signal Process.}, vol. 65, no. 15, pp. 4117--4131, 2017.

\bibitem{lagoudakis2003least}
M.~G. Lagoudakis and R.~Parr,
\newblock ``Least-squares policy iteration,''
\newblock {\em J. Mach. Learning Res.}, vol. 4, no. Dec, pp. 1107--1149, 2003.

\bibitem{geramifard2013tutorial}
A.~Geramifard et~al.,
\newblock ``A tutorial on linear function approximators for dynamic programming and reinforcement learning,''
\newblock {\em Foundations and Swerves{\textregistered} in Machine Learning}, vol. 6, no. 4, pp. 375--451, 2013.

\bibitem{agarwal2020flambe}
A.~Agarwal, S.~Kakade, A.~Krishnamurthy, and W.~Sun,
\newblock ``{FLAMBE}: Structural complexity and representation learning of low rank {MDP}s,''
\newblock in {\em Conf. Neural Inform. Process. Syst.}, 2020, vol.~33, pp. 20095--20107.

\bibitem{rozada2024tensor}
S.~Rozada, S.~Paternain, and A.~G. Marques,
\newblock ``Tensor and matrix low-rank value-function approximation in reinforcement learning,''
\newblock {\em {IEEE} Trans. Signal Process.}, vol. 72, pp. 1634--1649, 2024.

\bibitem{rozada2025solving}
S.~Rozada, J.~L. Orejuela, and A.~G. Marques,
\newblock ``Solving finite-horizon {MDP}s via low-rank tensors,''
\newblock {\em arXiv preprint arXiv:2501.10598}, 2025.

\bibitem{ormoneit2002kernel}
D.~Ormoneit and S.~Sen,
\newblock ``Kernel-based reinforcement learning,''
\newblock {\em Mach. Learn.}, vol. 49, no. 2, pp. 161--178, 2002.

\bibitem{akiyama2025nonparametric}
Y.~Akiyama and K.~Slavakis,
\newblock ``Nonparametric {B}ellman mappings for value iteration in distributed reinforcement learning,''
\newblock {\em arXiv preprint arXiv:2503.16192}, 2025.

\bibitem{horn2012matrix}
R.~A. Horn and C.~R. Johnson,
\newblock {\em Matrix Analysis},
\newblock Cambridge University Press, 2012.

\bibitem{sutton2018reinforcement}
R.~S. Sutton,
\newblock {\em Reinforcement Learning: {A}n Introduction},
\newblock A Bradford Book, 2018.

\bibitem{geist2013algorithmic}
M.~Geist and O.~Pietquin,
\newblock ``Algorithmic survey of parametric value function approximation,''
\newblock {\em {IEEE} Trans. Neural Netw. Learning Syst.}, vol. 24, no. 6, pp. 845--867, 2013.

\bibitem{choi2006generalized}
D.~Choi and B.~Van~Roy,
\newblock ``A generalized {K}alman filter for fixed point approximation and efficient temporal-difference learning,''
\newblock {\em Discrete Event Dyn. Syst.}, vol. 16, no. 2, pp. 207--239, 2006.

\bibitem{asadi2023td}
K.~Asadi, S.~Sabach, Y.~Liu, O.~Gottesman, and R.~Fakoor,
\newblock ``{TD} convergence: {A}n optimization perspective,''
\newblock in {\em Conf. Neural Inform. Process. Syst.}, 2023, vol.~36, pp. 49169--49186.

\bibitem{nowlan1992simplifying}
Steven~J. Nowlan and Geoffrey~E. Hinton,
\newblock ``Simplifying neural networks by soft weight-sharing,''
\newblock {\em Neural Comput.}, vol. 4, no. 4, pp. 473--493, 1992.

\bibitem{brockman2016openai}
G.~Brockman et~al.,
\newblock ``Open{AI} {G}ym,''
\newblock {\em arXiv preprint arXiv:1606.01540}, 2016.

\bibitem{towers2024gymnasium}
M.~Towers et~al.,
\newblock ``Gymnasium: {A} standard interface for reinforcement learning environments,''
\newblock {\em arXiv preprint arXiv:2407.17032}, 2024.

\end{thebibliography}
